\documentclass{article}

 \usepackage[preprint]{neurips_2026}

\usepackage[utf8]{inputenc} 
\usepackage[T1]{fontenc}    
\usepackage{url}            
\usepackage{booktabs}       
\usepackage{amsfonts}       
\usepackage{nicefrac}       
\usepackage{microtype}      
\usepackage{xcolor}         

\usepackage{tikz}
\usepackage{subcaption}
\usetikzlibrary{shapes.geometric}

\usepackage{fontawesome5}
\usepackage{custom, customboxes}

\title{Causal learning with the invariance principle}

\author{%
  Francesco Montagna \\
Institute of Science and Technology Austria\\
\texttt{francesco.montagna@ist.ac.at} \\
  \And
  Francesco Locatello \\
Institute of Science and Technology Austria
}

\begin{document}

\maketitle

\begin{abstract}
\looseness-1Causal discovery, the problem of inferring the direction of causality, is generally ill-posed. We use the language of structural causal models (SCM) to show that assuming that the causal relations are acyclic and invariant across multiple environments (e.g., the way minimum wage affects employment rate is stable across different geographical regions), \textit{only} two auxiliary environments are sufficient to infer the causal graph for arbitrary nonlinear mechanisms. Moreover, we demonstrate that this implies identifiability of the SCM functional mechanisms: as a corollary, we show that \textit{two} auxiliary environments are sufficient to guarantee correct counterfactual inference. We empirically support our theoretical results on synthetic data.
\end{abstract}

\section{Introduction}
\looseness-1Causal discovery, the problem of estimating a causal graph from data, is a cornerstone of causal inference, necessary for inferring causal effects under interventions \citep{pearl2009causality}. However, it is {well known} to be ill-posed: {multiple causal models are generally compatible with the observed data distribution}, and the causal graph is \textit{not observationally identifiable}.

\looseness-1The way out is {to impose} additional constraints on the causal structure by collecting more data while still avoiding explicit interventions. This idea underlies many identifiability results in causal discovery 
\citep{peters2015invariant,perry2022causal,jalaldoust2025multidomain,montagna2026on},
independent component analysis \citep{hyvarinen2016time,hyvarinen2019nonlinear,gresele2019incomplete,lachapelle22disentanglement}, and related generalizations (e.g., causal representation learning, \citet{scholkopf21towards}). These constraints often come from auxiliary datasets, which we call \textit{environments}. This idea is also known in {the causal inference} literature under the name of \textit{natural experiments} \citep{titiunik2021natural}: to infer causality, we generally need a control group (a base environment, in our language) and a treatment group (an auxiliary environment). Unlike randomized controlled trials, where the researcher defines the experiment, in natural experiments one only collects observational data from different distributions: if the distribution shifts are subject to sufficient constraints, {pairing observational data suffices} for causal inference (e.g., \citet{card1994} shows that minimum wage affects employment rate assuming that the causal relation is stable across geographical regions). This is the idea at the heart of multi-environment causal inference.

\looseness-1Our work formalizes this idea and provides sufficient conditions for the identifiability of causal graphs for arbitrary nonlinear SCMs from \textit{only} two auxiliary environments: the only requirements {we} impose on the data-generating process {are} that the causal graph is acyclic, and the function $\f: \bfS \mapsto \bfX$ mapping {latent variables $\bfS$ to observed variables $\bfX$} is invertible (which is always assumed in any tractable causal discovery problem). 
{The requirement} of \textit{only} two environments for such a general class of SCMs is remarkable in {light of prior work}, which either {requires} the number of {auxiliary} environments to scale with the number of nodes in the graph \citep{monti2019nonsens} or makes strong distributional assumptions such as Gaussianity of the noise terms \citep{montagna2026on}). Moreover, we show that acyclicity and two auxiliary environments are sufficient for the identifiability of counterfactual outcomes, the highest rung of the causal hierarchy.

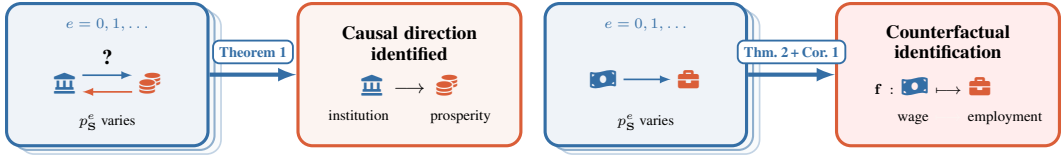
\begin{figure*}[t]
\centering
\begin{subfigure}[b]{0.49\linewidth}
\centering
\resizebox{\linewidth}{!}{%
\begin{tikzpicture}[font=\small, >=latex]
\definecolor{Cenv}{RGB}{52,110,165}
\definecolor{Cone}{RGB}{52,110,165}
\definecolor{Ctwo}{RGB}{218,94,56}
\definecolor{Ccor}{RGB}{131,82,168}

\node[rectangle, rounded corners=7pt, draw=Cenv!30, fill=Cenv!3,
      line width=0.8pt, minimum width=3.4cm, minimum height=2.4cm]
     at (0.22, -0.14) {};
\node[rectangle, rounded corners=7pt, draw=Cenv!55, fill=Cenv!5,
      line width=1.0pt, minimum width=3.4cm, minimum height=2.4cm]
     at (0.11, -0.07) {};
\node[rectangle, rounded corners=7pt, draw=Cenv, fill=Cenv!8,
      line width=1.5pt, minimum width=3.4cm, minimum height=2.4cm]
     (envbox) at (0,0) {};
\node[font=\scriptsize\bfseries, text=Cenv] at (0, 0.85) {$e = 0,1,\ldots$};
\node[font=\normalsize\bfseries, text=black] at (0, 0.30) {?};
\node[font=\normalsize] (inst) at (-0.72, -0.15) {\color{Cenv}\faLandmark};
\node[font=\normalsize] (money) at (0.72, -0.15) {\color{Ctwo}\faCoins};
\draw[->, thick, Cenv, transform canvas={yshift=4pt}]
    (inst.east) -- (money.west);
\draw[->, thick, Ctwo, transform canvas={yshift=-4pt}]
    (money.west) -- (inst.east);
\node[font=\scriptsize, text=black] at (0, -0.80) {$p^e_{\mathbf{S}}$ varies};

\node[rectangle, rounded corners=7pt, draw=Ctwo, fill=Ctwo!7,
      line width=1.5pt, minimum width=3.4cm, minimum height=2.4cm,
      text width=3.2cm, align=center, inner sep=8pt]
     (resbox) at (5.1, 0)
{
  \small\textbf{Causal direction}\\
  \small\textbf{identified}\\[6pt]
  {\normalsize\color{Cenv}\faLandmark}
  \;$\longrightarrow$\;
  {\normalsize\color{Ctwo}\faCoins}\\[3pt]
  \scriptsize institution \hspace{2em} prosperity
};

\draw[->, line width=2pt, Cone]
    (envbox.east) -- (resbox.west)
    node[midway, above=5pt, rectangle, rounded corners=3pt,
         draw=Cone, fill=white, line width=0.9pt,
         font=\scriptsize\bfseries, text=Cone, inner sep=3pt]
    {Theorem~1};
\end{tikzpicture}%
}
\caption{\textbf{Colonial institutions and long-run prosperity.}
Different European colonies serve as environments modulating the background conditions $p_{\bfS}^e$, while
$\mathbf{f}: \textnormal{institutions } \mapsto \textnormal{ prosperity}$ is invariant.
\citet{acemoglu2001} show that the inclusiveness of institutions towards indigenous people affects long-term prosperity.}
\label{fig:colony}
\end{subfigure}
\hfill
\begin{subfigure}[b]{0.49\linewidth}
\centering
\resizebox{\linewidth}{!}{%
\begin{tikzpicture}[font=\small, >=latex]
\definecolor{Cenv}{RGB}{52,110,165}
\definecolor{Cone}{RGB}{52,110,165}
\definecolor{Ctwo}{RGB}{218,94,56}
\definecolor{Ccor}{RGB}{131,82,168}

\node[rectangle, rounded corners=7pt, draw=Cenv!30, fill=Cenv!3,
      line width=0.8pt, minimum width=3.4cm, minimum height=2.4cm]
     at (0.22, -0.14) {};
\node[rectangle, rounded corners=7pt, draw=Cenv!55, fill=Cenv!5,
      line width=1.0pt, minimum width=3.4cm, minimum height=2.4cm]
     at (0.11, -0.07) {};
\node[rectangle, rounded corners=7pt, draw=Cenv, fill=Cenv!8,
      line width=1.5pt, minimum width=3.4cm, minimum height=2.4cm]
     (envbox) at (0,0) {};
\node[font=\scriptsize\bfseries, text=Cenv] at (0, 0.85) {$e = 0,1,\ldots$};
\node[font=\normalsize] (wage) at (-0.72, -0.08) {\color{Cenv}\faMoneyBillWave};
\node[font=\normalsize] (empl) at (0.72, -0.08) {\color{Ctwo}\faBriefcase};
\draw[->, thick, Cenv] (wage.east) -- (empl.west);
\node[font=\scriptsize, text=black] at (0, -0.80) {$p^e_{\mathbf{S}}$ varies};

\node[rectangle, rounded corners=7pt, draw=Ctwo, fill=red!7,
      line width=1.5pt, minimum width=3.4cm, minimum height=2.4cm,
      text width=3.2cm, align=center, inner sep=8pt]
     (resbox) at (5.1, 0)
{
  \small\textbf{Counterfactual}\\
  \small\textbf{identification}\\[8pt]
  \scriptsize $\mathbf{f}:\text{\normalsize\color{Cenv}\faMoneyBillWave}\longmapsto\text{\normalsize\color{Ctwo}\faBriefcase\hspace{1.5em}}$\\[3pt]
  \hspace{3em}wage \textcolor{Ctwo!7}{$\longmapsto$} employment
};

\draw[->, line width=2pt, Cone]
    (envbox.east) -- (resbox.west)
    node[midway, above=5pt, rectangle, rounded corners=3pt,
         draw=Cone, fill=white, line width=0.9pt,
         font=\scriptsize\bfseries, text=Cone, inner sep=3pt]
    {Thm.~2\,+\,Cor.~1};
\end{tikzpicture}%
}
\caption{\textbf{Minimum wage and empolyment rate.}
In \citet{card1994}, New Jersey and Pennsylvania act as two
environments: given the causal graph, a minimum-wage increase in New Jersey, but not in
Pennsylvania, creates exogenous variation that identifies the invariant causal effect $\mathbf f$.}
\label{fig:wage}
\end{subfigure}
\caption{Natural experiments as instances of our multi-environment framework.}
\label{fig:natural_experiments}
\end{figure*}

\subsection{Summary of contributions}
From a theoretical perspective, our paper consists of three main contributions. We (informally) summarize them below. 

\begin{callout}
    \begin{theorem}[Informal version of causal graph identifiability]\label{thm:full_graph_identifiability}
    Given a structural causal model $\bfX = \f(\bfS)$ with bijective function $\f$, the causal graph is identifiable from \emph{only} two sufficiently different auxiliary environments.
    \end{theorem}
\end{callout}

\begin{callout}
    \begin{theorem}[Informal version of ICA identifiability]\label{thm:ica_identifiability}
    Given a structural causal model $\bfX = \f(\bfS)$ with $\f$  bijective, and two sufficiently different auxiliary environments, the function $\f$ is identifiable up to trivial indeterminacies.
    \end{theorem}
\end{callout}
As a corollary of \cref{thm:ica_identifiability}, we have the following result:
\begin{callout}
    \begin{corollary}[Informal corollary of \cref{thm:ica_identifiability}]\label{thm:counterfactual}
        A structural causal model estimated from three sufficiently distinct environments is enough for counterfactual inference.
    \end{corollary}
\end{callout}

Empirically, based on our theory, we introduce MCD (Multienvironment Causal Discovery), the first algorithm 
that, in the population limit, provably returns the correct causal graph for \textit{arbitrary} noise distribution and causal mechanisms, given \textit{only} two auxiliary environments.


\section{Preliminaries}
Throughout this paper, we take the perspective that a structural causal  model with independent latent variables is an independent component analysis (ICA) model of the form 
\begin{equation}\label{eq:ica}
    \begin{split}
        &\bfX = \f(\bfS), \quad \bfX \in \mathcal X \subseteq \R^d\\
        &p_{\bfS}(\s) = \prod_{i=1}^d p(s_i), \quad \forall \s \in \mathcal S \subseteq \R^d, \\
        &\exists P \in \R^{d \times d} \textnormal{ permutation s.t. } P J_{\f} P^T \textnormal{ is lower triangular.}
    \end{split}
\end{equation}
\looseness-1Assume $p_{\bfS}$ strictly positive density. The permutation constraint distinguishes SCMs from standard ICA models: without it, \cref{eq:ica} reduces to a general ICA model; with it, the Jacobian $J_{\f}$ is triangularizable, which is equivalent to acyclicity of $\mathcal G$.
(Uppercase letters, e.g. $\bfS$, are used for random variables, lowercase letters for their realization.) This perspective is quite established in the literature 
\citep{monti2019nonsens,hyvarinen2023identifiability,reizinger2023jacobianbased,montagna2026on}. While independent component analysis aims to recover $\f$ and $p_{\bfS}$ (assuming a density exists) from observations of $\bfX$, causal discovery is a much simpler problem: the goal here is to infer the causal graph that is imposed on $\bfX$ variables by $\f$. In particular, an SCM assumes a set of structural causal equations  
\begin{equation}\label{eq:scm}
X_i := F_i(\mathbf X_{\Parents_i}, S_i), \hspace{.5em}\forall i = 1,...,d,
\end{equation}
\looseness-1 where $\mathbf X_{\Parents_i}$ are the causes of $X_i$, specified by a directed acyclic graph (DAG) $\mathcal G$ with nodes $\bfX$; $\Parents_i \subset \set{1,...,d}$ are the parent indices of $X_i$. The $F_i$ functions are the \textit{causal mechanisms}. We further assume no latent common causes, no selection bias, and causal minimality. Recursively applying \cref{eq:scm} gives exactly \cref{eq:ica}.

\paragraph{Causal order.} Intuitively, the causal order is the direction of causality, and is the main object of interest in causal discovery. For example, in $X_1 \to X_2$, the only valid causal order is $\set{1,2}$; for a graph $X_2 \leftarrow X_1 \to X_3$, $\set{1, 2, 3}$ or $\set{1, 3, 2}$ are \textit{valid} causal orders; $\set{2, 3, 1}$ is not. 

The causal order of a graph is the fundamental object of inference in causal discovery. When the causal order can be recovered from observations of $\bfX$, we say that the causal graph is \textit{identifiable}. In general, this requires access to a control and a treatment group: in this work, multiple environments serve as naturally arising control and treatment groups in passively collected data.
%

\subsection{Base and auxiliary environments as control and treatment groups in natural experiments}
Consider the SCM of \cref{eq:ica}. We call this the \textit{base} environment. An \textit{auxiliary} environment is the structural causal model:
\begin{equation}
    \begin{split}
        &\bfX^e = \f(\bfS^e),\quad p^e_{\bfS}(\s) = \prod_{i=1}^d p^e(s_i)
    \end{split}
\end{equation}
The idea is that we have access to the distribution of $\bfX$ under several shifts corresponding to environments, and postulate that these shifts are not arbitrary and follow an invariance principle.
\begin{callout}
    \textbf{The invariance principle.} An environment is characterized by
    \begin{enumerate}
        \item \textit{Invariant} causal mechanisms $\f$;
        \item $p_{\bfS}^0 \neq p_{\bfS}^e$, with $0$ denoting the base environment, and $e=1,...,|\mathcal E|$,
    \end{enumerate}
    where $\mathcal E \subset \mathbb N$ collects the indices of all auxiliary environments.
\end{callout}
\begin{definition}[Informal definition of identifiability for causal discovery]\label{def:cd_identifiability}
   \looseness-1The causal graph $\mathcal G$ is \textit{identifiable} if the collection of observed distributions $(p_{\bfX}^e)_{e=0}^{|\mathcal E|}$ uniquely determines the causal graph $\mathcal G$.
\end{definition}

Stronger notions are possible: more importantly, we have identifiability in the ICA sense. 

\begin{definition}[Informal definition of identifiability for ICA]\label{def:ica_identifiability}
   An ICA model is \textit{identifiable} if the collection of observed distributions $(p_{\bfX}^e)_{e=0}^{|\mathcal E|}$ uniquely determines $\f$ and $p_{\bfS}^e$ up to trivial indeterminacies.
\end{definition}

We now state the main research problem.

\begin{emptyroundbox}
    \looseness-1\textbf{Problem definition.} We characterize when a causal graph $\mathcal G$ is identifiable from the fewest possible environments. Moreover, we find the minimal set of assumptions under which the causal mechanisms and noise terms are also identifiable from a few environments in the ICA sense. 
\end{emptyroundbox}

\paragraph{Environments as natural experiments.} \looseness-1The notion of environment was popularized in causal discovery by \citet{peters2015invariant}, but the underlying idea is older: the base environment is a control group and an auxiliary environment is a treatment group. When such groups arise naturally rather than being engineered, we speak of \textit{natural experiments} \citep{titiunik2021natural}: our framework is precisely a formalization of this concept with SCMs. Notable instances of causal inference with natural experiments are presented in \citet{acemoglu2001} and \citet{card1994}, and illustrated in \cref{fig:colony} and \cref{fig:wage}, serving as real-world motivating examples for our theory.

Next, we demonstrate our main theoretical claims, namely, that \textit{only} two auxiliary environments are sufficient to identify the causal graph and the mixing function of any acyclic and bijective SCM. 

\section{Theory of identifiability of bijective causal models}
This section is devoted to provide the intuition and formalize the theoretical results of \cref{thm:full_graph_identifiability}, \cref{thm:ica_identifiability}, and \cref{thm:counterfactual}. In particular, we show that causal discovery for arbitrary invertible SCMs can be achieved with \textit{only} two auxiliary environments; moreover, we find that this is equivalent to ICA identifiability when the structural causal model is assumed acyclic.


\subsection{Causal structure learning from multiple environments}

To develop our theory, we need to constrain the class of structural causal models we consider. Intuitively, the main restrictive assumption that we make is invertibility of the causal mechanism $\mathbf f$. 

\begin{assumption}[Invertibility]\label{ass:diffeomorphism}
    $\mathbf f$ is a global diffeomorphism.
\end{assumption}
%
\begin{assumption}[Faithfulness]\label{ass:faithfulness}
    Consider $\mathbf Z \in \R^m$ and $\mathbf V \in \R^n$ random vectors, and the acyclic structural equation model $Z_i := g_i(\mathbf V_{\Parents_i})$, and a DAG $\mathcal G$ with nodes $\mathbf Z, \mathbf V$ and edges from $\mathbf V_{\Parents_i}$ to $Z_i$. We assume faithfulness of $p_{\mathbf Z}$ density to $\mathcal G$, i.e.:
    $$
    \mathbf Z_{A} \indep \mathbf Z_{B} | \mathbf Z_{C} \implies \mathbf Z_{A} \indep_d \mathbf Z_{B} | \mathbf Z_{C},
    $$
    where $\indep_d$ denotes classical d-separation on $\mathcal G$, and $\indep$ denotes probabilistic conditional independence. $A, B, C \subseteq \set{1,...,m}$ are non-overlapping sets of indices.
\end{assumption}
\begin{assumption}[Sufficient variability]\label{ass:sufficient_variability} We have access to at least two auxiliary environments $e=1,2$. Let
$\Delta(0, e, \x) := \nabla \left(\log p(\x)^0 - \log p(\x)^e\right)$. Then:
\begin{enumerate}[(i)]
    \item $\Delta(0,2,\x)_i \neq 0 \: \text{a.s.}$, $i=1,...,d$.
    \item For every non-sink $X_i$, there exists $j \neq i$ such that the vectors $\big(\Delta(0,1,\x)_i,\Delta(0,2,\x)_i\big)$ and $\big(\partial_{s_j}\Delta(0,1,\x)_i,\partial_{s_j}\Delta(0,2,\x)_i\big)$ are not collinear on a positive-measure set.
\end{enumerate}
\end{assumption}

\paragraph{Comparison with assumptions in the literature.}\cref{ass:diffeomorphism} is ubiquitous when studying identifiability of causal graphs and beyond (in ICA and causal representation learning): causal graph identifiability from pure observations requires additivity of the noise terms \citep{hoyer2008anm} or variations of it, as in the post-nonlinear and location scale noise models \citep{zhang2009pnl,immer2022on}. All these parametric assumptions directly imply diffeomorphism \citep{dominguez2023on} (hence, our \cref{ass:diffeomorphism}); instead, the viceversa is not true, making our hypothesis strictly less demanding (for example, the celebrated additive noise models of \citet{hoyer2008anm} equivalently assume that all the eigenvalues of the Jacobian of the mixing function are equal to $1$, a very strong requirement) . In the multienvironment setting, bijectivity is also a common requirement \citep{peters2015invariant,heinze2018invariant,ghassami2018multidomain,jalaldoust2025multidomain,montagna2026on}.  
\cref{ass:faithfulness} of faithfulness is a common requirement in causal discovery, and known to be generic (almost always satisfied) at the population level; hence, it is mild and non-restrictive. Our formulation for arbitrary structural equation models is just a variation of the usual faithfulness for standard causal models (including SCMs). Genericity of \cref{ass:faithfulness} is discussed in detail in \cref{app:faithfulness}.
Finally, \cref{ass:sufficient_variability} of sufficient variability comes in different forms but is standard in multienvironment identifiability arguments in causal discovery \citep{montagna2026on,chevalleyderiving}, ICA \citep{hyvarinen2016time,lachapelle22disentanglement}, and causal representation learning \citep{ahuja2023multi,varici2025score,ng2025causal}, and just excludes pathological environment choices. Part \textit{(i)} simply requires a treatment group that is sufficiently different in the entire domain. Part \textit{(ii)} has a less immediate interpretation: it is a geometric requirement asking that the first-order differences of densities are not aligned with second-order differences of the densities. Intuitively, this should be true, as there is no geometric reason to believe that such alignment should occur. This can be proven more formally for quite a broad class of SCMs: in \cref{app:sufficient_variability}, we show that for SCMs with a finite number of parameters in the Euclidean space, whose mixing function is analytic, the set of density parameters that lead to violations of the assumption lives in a  Lebesgue-zero measure subset of the Euclidean space. Such geometric requirements on the class of allowed environments are present for any result under the umbrella of \textit{identifiability with multiple environments} that we mentioned, suggesting that it is somehow unavoidable to discard pathological cases, as we find ourselves. Finally, we note that \cref{ass:sufficient_variability} requires two auxiliary environments. All the results that follow can be generalized to the case where we have access to an arbitrary number of environments: in which case, we partition the auxiliary environments in non-overlapping groups $\mathcal E_1, \mathcal E_2$, and define $\Delta(0,e,\x) := \nabla\left(\log p^0(\x) - \sum_{\epsilon \in \mathcal E_e} \log p^\epsilon(\x) \right)$.

Next, we show that given our modeling assumptions, we can demonstrate identifiability of the causal graph. We provide a step-by-step derivation of the results with two variables via an example. 

\begin{example}\label{example:bivariate_sink}
    Consider a bivariate structural causal model $x_2 := f_2(x_1, s_2)$, such that $\mathbf f: s_1, s_2 \mapsto s_1, x_2$. Let $\x = \mathbf f(\mathbf s)$, where by injectivity of $\mathbf f$, $\mathbf s$ is uniquely determined. Given three environments $p^e$, $e \in \set{0, 1, 2}$, we write the score function (gradient of the log-likelihood) of $\mathbf x$:
    \begin{equation}
        \nabla \log p^e(\mathbf x) = J^T_{\mathbf f^{-1}}(\mathbf x) \nabla \log p^e(\mathbf s) + \nabla \log |J_{\f^{-1}}(\mathbf x)|,
    \end{equation}
    where $|J_{\f^{-1}}(\mathbf x)|$ denotes the determinant. The latter expression can be easily verified via the change of variable formula for densities. We note that the gradient of the log-det term (the volume arising from the change of variable formula) does not depend on the index of the environment. Then, if we take the difference of the score of the base and auxiliary environments, it cancels out as follows:
    \begin{equation}
        \Delta(0, e, \mathbf x) := \nabla \left(\log p^0(\mathbf x) - \log p^e(\mathbf x)\right) = J^T_{\mathbf f^{-1}}(\mathbf x) \nabla \left(\log p^0(\mathbf s) - \log p^e(\mathbf s)\right), 
    \end{equation}
    that is, the derivative of the log-odds ratio does not depend on the  volume term. Let's unpack the $\Delta$ term componentwise:
    \begin{equation} \label{eq:deltas}
        \begin{split}
            &\Delta(0, e, \mathbf x)_1 = \partial_{s_1}\left(\log p^0(s_1) - \log p^e(s_1)\right) + \partial_{x_1} s_2 \partial_{s_2}\left(\log p^0(s_2) - \log p^e(s_2)\right) \\
            &\Delta(0, e, \mathbf x)_2 = \partial_{x_2} s_2 \partial_{s_2}\left(\log p^0(s_2) - \log p^e(s_2)\right). 
        \end{split}
    \end{equation}
    We know, by assumption, that the correct causal model (hence, the SCM with correct causal direction) must have mutually independent sources. Previous literature exploited this hypothesis in the context of additive noise models. In particular, \citet{montagna2025score} shows that for observationally identifiable SCMs, the score of a sink node $X_l$ is a \textit{only} a function of the source node $S_l$, and uses this intuition to demonstrate identifiability of the causal order. However, this is not the case for the second equality in \eqref{eq:deltas}, due to the multiplier $\partial_{x_2} s_2 = \partial_{x_2} \mathbf f^{-1}_{2}(x_1, x_2)$. To recover the desired situation of the score (a function of the score, in our case) of the sink node $X_2$ depending only on $S_2$, we note that the term $\partial_{x_2} s_2$ is uniquely determined by the form of $\f$, the causal mechanisms, that is invariant with respect to the environment index. We exploit this fact by taking the ratio
    $$
    \bfR(\s) := \frac{\Delta(0, 1, \mathbf \f(\s))}{\Delta(0, 2, \f(\s))},
    $$
    where
    \begin{equation}\label{eq:example_ratios}
        \begin{split}
        &R_1 := \bfR_1(\s) = \frac{\partial_{s_1}\left(\log p^0(s_1) - \log p^1(s_1)\right) + \partial_{x_1} s_2 \partial_{s_2}\left(\log p^0(s_2) - \log p^1(s_2)\right)}{\partial_{s_1}\left(\log p^0(s_1) - \log p^2(s_1)\right) + \partial_{x_1} s_2 \partial_{s_2}\left(\log p^0(s_2) - \log p^2(s_2)\right)} \\
        &R_2 := \bfR_2(\s) = \frac{\partial_{s_2}\left(\log p^0(s_2) - \log p^1(s_2)\right)}{\partial_{s_2}\left(\log p^0(s_2) - \log p^2(s_2)\right)}.
        \end{split}
    \end{equation}
    Inutively, the above relation defines a structural equation model $\mathbf g: \bfS \mapsto \bfR$ where:
    \begin{equation*}
    \begin{split}
        &R_1 := \mathbf{g}_1(S_1, S_2)\\
        &R_2 := \mathbf{g}_2(S_2).
    \end{split}
    \end{equation*}
    The above is true unless, due to pathological algebraic cancellations, we have $\partial_{s_i} R_1(\s) = 0$ almost surely for $i=1$ or $2$. These are ruled out by \cref{ass:sufficient_variability} of sufficient variability.
    Hence, we obtain the ratio $R_2$ depending only on $S_2$, and $R_1$ depending on the noise vector $S_1, S_2$. Then, there is a structural equation model
    $$
    \bfX := \f(\bfS), \quad \bfR := \mathbf g(\bfS)
    $$
    specified by \cref{eq:ica} and \cref{eq:example_ratios},
    with the DAG of \cref{fig:measurement_model}. A straightforward application of the \cref{ass:faithfulness} of faithfulness over the structural equation models verifies that:
    \begin{align}
        &R_2 \indep X_1, \quad\! X_1 \sim  p^e_{X_1}\label{eq:r2_indep_x1} \\
        & R_1 \notindep X_2  \quad\! X_2 \sim p^e_{X_2},\label{eq:r1_notindep_x2}
    \end{align}
    We conclude that for generic SCMs, simply testing the \emph{unconditional} independence of $\bfR$ and $\bfX$ entries yields a criterion for identifying the causal order.
\end{example}

\begin{figure}
    \centering
    \begin{tikzpicture}[
    scale=0.7, 
    every node/.style={circle, inner sep=0pt, font=\small},
    s_node/.style={draw=black, fill=black, text=white, minimum size=0.6cm},
    xr_node/.style={draw=black, fill=white, text=black, minimum size=0.6cm},
    >=latex, 
    line width=0.5pt 
    ]
    
    \node[s_node] (s1) at (-2.2, 2.2) {$S_1$};
    \node[s_node] (s2) at (2.2, 2.2) {$S_2$};
    
    \node[xr_node] (r1) at (0, 1.2) {$R_1$}; 
    
    \node[xr_node] (x1) at (-2.2, 0) {$X_1$};
    \node[xr_node] (x2) at (0, 0) {$X_2$};
    \node[xr_node] (r2) at (2.2, 0) {$R_2$};
    
    \draw[->] (s1) -- (x1);
    \draw[->] (s1) -- (x2);
    \draw[->] (s1) -- (r1);
    \draw[->] (s2) -- (r1);
    \draw[->] (s2) -- (x2);
    \draw[->] (s2) -- (r2);
    
    \end{tikzpicture}
        \caption{Graph for the structural equation model of \cref{example:bivariate_sink}. $X_1, X_2, R_1, R_2$ are observable measurements of the latent variables $S_1, S_2$. The resulting graph is a DAG: each observed measurement is a function of its parent nodes in the structural equations defining the model. By d-separation, we have $R_1  \notindep_d X_2$ and $R_2 \indep_d X_1$, according to the analysis in the example, which distinguishes the causal order on the SCM over $X_1, X_2$.}
        \label{fig:measurement_model}
\end{figure}
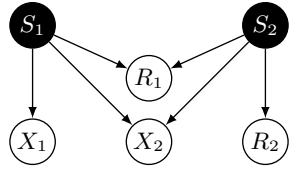

The above example provides the key technical steps to our main result. 

\begin{restatable}[Formalization of \cref{thm:full_graph_identifiability}]{theorem*}{CdIdentifiability}
Consider the SCM of \cref{eq:ica}. Let \cref{ass:diffeomorphism} and \cref{ass:faithfulness} be satisfied. Then, given at least two auxiliary environments satisfying \cref{ass:sufficient_variability} of sufficient variability for each node, the causal graph is identifiable.
\end{restatable}
The proof is given in \cref{app:cd_proof}.

\begin{callout}
    \paragraph{Discussion.} Under suitable assumptions, \textit{only} two auxiliary environments are sufficient to identify the causal graph of an SCM with arbitrary bijective mechanisms. This is in stark contrast with multi-environment identifiability results in ICA, where the number of environments necessarily linearly scales with the number of nodes. Similar findings were recently presented in \citet{montagna2026on}, however, under the restrictive assumption of Gaussianity of the SCM noise terms and on the nature of the environments.
\end{callout}
The above theorem, while presenting (to the best of our knowledge) the most general identifiability result in the theory of causality with SCMs, is under the constraint that all noise terms are subject to shifts in at least two environments (\cref{ass:sufficient_variability}).  From a theoretical perspective, this is acceptable, as studying the conditions of full graph identifiability is a well-known problem of interest, and widely researched in the literature \citep{shimizu2006alinear,hoyer2008anm,zhang2009pnl,immer2022on,montagna2026on}; however, practical causal discovery should consider what is identifiable of the graph given a set of datasets under the least possible assumptions about their generating process. In the appendix, we show that violations of \cref{ass:sufficient_variability} can be detected from the data---specifically, we can verify if some of the latent variables are not subject to interventions from statistics of the observed causal variables. Characterizing the nature of the equivalence class that is still identifiable when only a subset of latents are subject to variation in the auxiliary environments is an important research direction that we open to future work.

\paragraph{Out of distribution generalization.}Given $p_{\bfX}$ and the causal graph $\mathcal G$, Pearl's backdoor adjustment provides a recipe for provably correct inference at the interventional level of the causal ladder  \citep{pearl2009causality}. Hence, given the result in \cref{thm:full_graph_identifiability}, two auxiliary environments allow for trustworthy  OOD generalization over data generated under interventions on the SCM. Next, we show some consequences that are even more striking: the DAG assumption, a strong yet widely adopted restriction in causal discovery, is enough to elevate these conclusions to the counterfactual level and beyond.


\subsection{ICA identifiability and counterfactual inference}\label{sec:ica}
Consider the SCM of \cref{eq:ica}, that is, an ICA model with Jacobian that can be triangularized by a permutation matrix. Under the assumptions of diffeomorphism and faithfulness, we can show that $\f$ is identifiable in the ICA sense, i.e. up to elementwise reparametrization of the sources. This is a consequence of the fact that the DAG assumption implies strong constraints on the form of $\f$, which is enough to elevate identifiability in the sense of causal discovery to identifiability in the ICA sense.

\begin{restatable}[Formalization of \cref{thm:ica_identifiability}]{theorem*}{IcaIdentifiability} Consider the ICA model \cref{eq:ica} for $\bfX = \f(\bfS)$. Let Assumptions \ref{ass:diffeomorphism} and \ref{ass:faithfulness} to hold. Further, assume that the domain of $\mathbf f: \mathcal S \to \mathcal X$ is connected. Then, given at least two auxiliary environments satisfying \cref{ass:sufficient_variability} $\f$ is identifiable in the ICA sense, i.e. up to an invertible and elementwise reparametrization of the sources.
\end{restatable}
A sketched and complete proofs are presented in \cref{app:ica_proof}.

An interesting consequence (independently found by \citet{nasr2023counterfactual}, assuming the graph is given) is that identifiability in the sense of ICA implies generalization at the counterfactual level. 
\begin{restatable}[Formalization of \cref{thm:counterfactual}]{corollary*}{CorCounterfactual}
    Consider the ICA model \cref{eq:ica} for $\bfX = \f(\bfS)$. Let Assumptions \ref{ass:diffeomorphism} and \ref{ass:faithfulness} to hold, and assume that the domain of $\mathbf f: \mathcal S \to \mathcal X$ is connected. Then, the SCM is counterfactually identifiable from the observational distribution $p_{\bfX}$ and two auxiliary environments satisfying \cref{ass:sufficient_variability}.
\end{restatable}
Proof can be found in \cref{app:cor_proof}. For details on counterfactual identifiability, we refer to \citet{peters2017elements} (Section 3.3) and \citet{pearl2019primer} (Section 4.2).

\paragraph{Discussion and summary of the theory.} Our theory supports surprisingly strong conclusions. \textit{(i)} from a causality perspective, for acyclic and invertible SCMs \textit{two} auxiliary environments are sufficient for out-of-distribution generalization to the counterfactual level of the ladder of causation. This is remarkable under the lenses of \textit{natural experiments}: passively observing \textit{only} two sufficiently different treatment groups is enough to correctly infer the effect of arbitrary treatment in the system of interest (note that we only require changes in the treated populations; the treatment type can be invariant between the auxiliary environments); moreover, it is sufficient to \textit{imagine} the effect of counterfactual treatments. Clearly, given the strength of this conclusion, one should be careful in imposing the  DAG-ness and causal sufficiency assumptions, which are responsible for the heavy-weight lifting in elevating observational information to the counterfactual level. \textit{(ii)} from a representation learning perspective, ICA models with triangular monotonic maps are identifiable from \textit{only} two environments. This is remarkable in the light of previous literature: without the triangularity constraint on the Jacobian $J_\f$ of \cref{eq:ica}, it is known that identifiability can be achieved only with a number of environments that scales with the number of nodes. Also in this case, the strong assumption on the structure of the Jacobian is responsible for this result, which remains nevertheless impressive being ICA a notably hard unsupervised-learning task.


\section{Multi-environment causal discovery from data: algorithm and experiments}
\label{sec:algorithm}
\looseness-1In this section, we translate our theory on the identifiability of graphs from 3 or more environments into a practical algorithm for inference; the goal is to validate our theory on synthetically generated data. 

\subsection{The MCD algorithm}
We propose a new method for causal discovery that can leverage multiple datasets from different environments, that we call MCD (Multienvironment Causal Discovery). Intuitively, our approach recursively finds sink nodes based on the logic illustrated in \cref{example:bivariate_sink} and adopted in the proof of our main theorem (\cref{thm:full_graph_identifiability}). A word of caution: our method mostly serves the purpose of empirically validating our theoretical claims, rather than being a serious proposal for causal discovery. In fact, as the logic of the algorithm follows that of the proof of \cref{thm:full_graph_identifiability}, it is not optimized for empirical performance, and we leave the effort of developing new, accurate methods based on our theory for future work. However, our synthetic experiments show that MCD can achieve remarkable performance, inferring the causal order better than random and than performative baselines for observational causal discovery, both on linear Gaussian data and on SCMs with arbitrary nonlinear mechanisms and up to $20$ nodes. 

\paragraph{A recursive approach to causal order estimation.} \looseness-1The focus of our method is finding the causal order. Once that is given, it is known how to recover the causal graph, without the need of special modelling assumptions. Given a sink-identification criterion, we can recursively find the causal order. For example, consider the graph $X_3 \leftarrow X_1 \rightarrow X_2$. With enough samples an a correct sink-identification criterion, we find that $X_2$ is a sink (if we were to find $X_3$ sink, the argument would be perfectly symmetric). Then, we know the position of $X_2$ in the causal order, and we can safely remove it from the graph. Next, we run the sink-identification criterion on $X_1 \to X_3$, and expect to find that $X_3$ is a sink. So, we know that the causal order $X_1, X_3, X_2$ is a valid one. This idea is central to our method. To identify sink nodes, we borrow from the logic of \cref{example:bivariate_sink}: given the vector of ratios $\mathbf R$, we know that $R_i \indep \mathbf X_{-i} \iff $ node $i$ is a sink (i.e., the $i^\textnormal{th}$ component of the ratio vector is independent of all causal variables except for $X_i$). In practice, however, the ratio is numerically unstable, which is why we replace it with a theoretically equivalent but algorithmically friendlier alternative. This, and a sketch of the workflow of MCD are presented in \cref{alg:mcd}. 

\looseness-1The key statistical estimation steps (involving score estimation) and MCD's  algorithmic complexity are presented in \cref{app:algorithm}. Next, we discuss the empirical performance of our method on synthetic data. We compare to several algorithms from the observational causal discovery literature.

\begin{algorithm}[t]
\small
\caption{MCD: recursive sink peeling from score ratios}
\label{alg:mcd}
\DontPrintSemicolon
\begin{spacing}{1.1}
\KwData{Datasets $\mathcal D^e=\{\mathbf x_n^e\}_{n=1}^{N_e}$, $e=0,\dots,E$; auxiliary environments split $\set{1,...,E}=\mathcal E_1\cup\mathcal E_2$}
\KwResult{Estimated causal order $\widehat \pi$}

$\mathcal A \leftarrow \set{1,...,d}$ \tcp*[r]{\footnotesize{Active variables}}
$\widehat \pi \leftarrow ()$\;

\While{$|\mathcal A|>0$}{
 $\widehat{\mathbf g}^e(\mathbf x_{\mathcal A})
    \approx
    \nabla_{\mathbf x_{\mathcal A}}
    \log p^e(\mathbf x_{\mathcal A}),
    \qquad e=0,\dots,E.$ \tcp*[r]{\footnotesize{Estimate the scores}}

    \medskip
${\Delta}^{(\ell)}(\mathbf x_{\mathcal A})
    \leftarrow
    \widehat{\mathbf g}^{0}(\mathbf x_{\mathcal A})
    -
    \sum_{e\in\mathcal E_\ell}
    \widehat{\mathbf g}^{e}(\mathbf x_{\mathcal A}),
    \qquad \ell=1,2.
    $ \tcp*[r]{\footnotesize{Compute score differences}}

    \medskip
    $\mathbf U \leftarrow \left[(\Delta^{(1)}_i, \Delta^{(2)}_i) \operatorname{ for } i \in 1,...,d  \right]$
    \tcp*[r]{\footnotesize{Statistics (stabler than ratio)}}

    \medskip
    $\mathbf U \leftarrow \left[\frac{\mathbf U_i}{||\mathbf{U}_i||} \operatorname{ for } i \in 1,...,d  \right]$ 
    
    \medskip
    $\widehat s = \arg \min_{i \in \mathcal A} \operatorname{HSIC}(\mathbf U_i, \x_{\mathcal A \setminus \set{i}})$ \tcp*[r]{\footnotesize{HSIC returns p-value; argmin returns sink index}}

    \medskip
    $\widehat \pi \leftarrow (\widehat s,\widehat \pi)$ \tcp*[r]{\footnotesize{Prepend the newly found sink}}

    \medskip
    $\mathcal A \leftarrow \mathcal A\setminus\{\widehat s\}$ \tcp*[r]{\footnotesize{Remove sink from active nodes}}
}

\Return{$\widehat \pi$}
\end{spacing}
\end{algorithm}

\subsection{Experiments: validation of theoretical claims and benchmarking}\label{sec:experiments}

In this section, we analyse the performance of MCD on synthetic data generated by SCMs notably non-identifiable from pure observations. 

\paragraph{Synthetic data generation.} \looseness-1Data are synthetically generated as follows: we consider Erdos-Renyi randomly generated graphs \citep{erdos1960on}, a common choice for causal discovery benchmarking, with $3,5,10,20$ nodes. For each pair of nodes in the graph, we have a probability of attachment (drawing an edge between the pair) of $\set{0.8, 0.5, 0.25, 0.25}$ (different values for all the graphs' dimensions). Noise terms follow a Gaussian distribution, with mean sampled in the interval $[-2, 2]$ and standard deviation in $[0.5, 4]$, which we allow to shift between environments. For each structural equation $X_i := F_i(\bfX_{\Parents_i, S_i})$, with uniform probability the function $F_i$ is either a linear transformation (resulting in a linear Gaussian SCM, notoriously non-identifiable from observations), or, else, a random nonlinearity: we sample the weights of a 4-layers Residual Neural Network from a standard Gaussian $\mathcal N(1,1)$, with $\tanh$ activations and normalization at the last layer, so that $X_i = \operatorname{ResNet}_i(\mathbf{X}_{\Parents_i}, S_i)$---the source and parents are concatenated in a unique input vector.  \cref{fig:resnet} displays some of the random nonlinear mechanisms we sampled. Datasets consists of $3000$ observations per environment. In \cref{app:statistical_efficiency} we empirically analyse the statistical efficiency of the method.

\paragraph{Metric.} \looseness-1As the focus of MCD is finding the causal order of the graph, we adopt the well-established topological divergence $D_{\textnormal{top}}$ \citep{rolland2022score}. It counts the number of edges a fully connected graph would miss due to a mistake in inferring the causal order. E.g., for a graph $X_1 \to X_2 \to X_3$, an estimated causal order $\set{1,3,2}$ would correspond to $D_{\textnormal{top}}=1$, because the edge $X_2 \to X_3$ is forbidden by such ordering. In our plots, we normalize the divergence by the total number of edges in the graph, so that it ranges from $0$ (best) to $1$ (worst).

\begin{figure}
    \centering
    \includegraphics[width=0.75\linewidth]{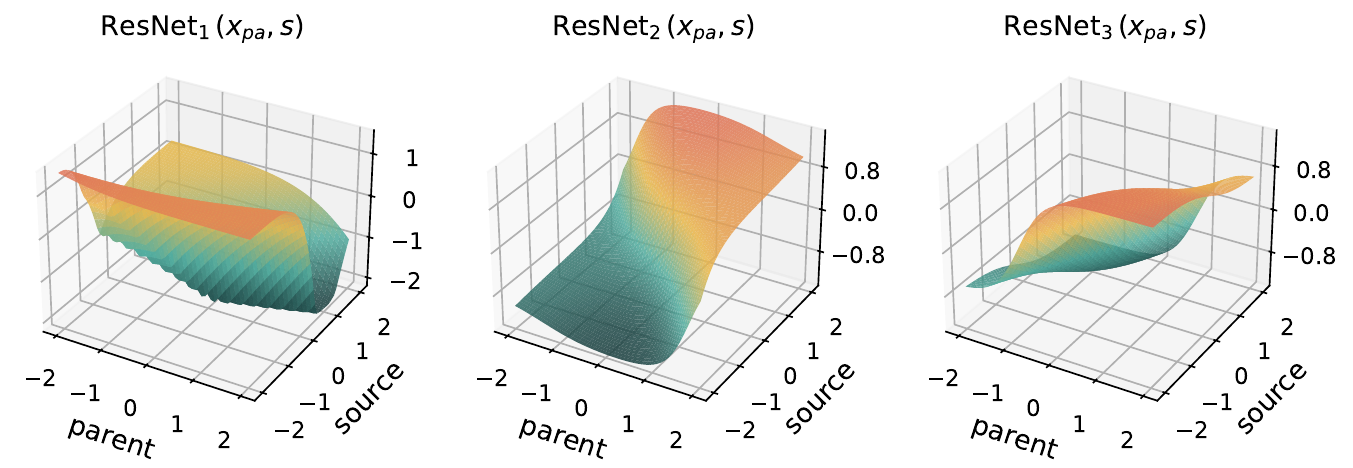}
    \caption{\looseness-1Samples of the random nonlinear mechanisms (one parent, one source) sampled from a residual neural network, adopted in our synthetic experiments. The function is nonlinear in the sources and sometimes even non-invertible (left figure), meaning we stress-test MCD beyond its assumptions.}
    \label{fig:resnet}
\end{figure}

\begin{figure}
    \centering
    \includegraphics[width=0.6\linewidth]{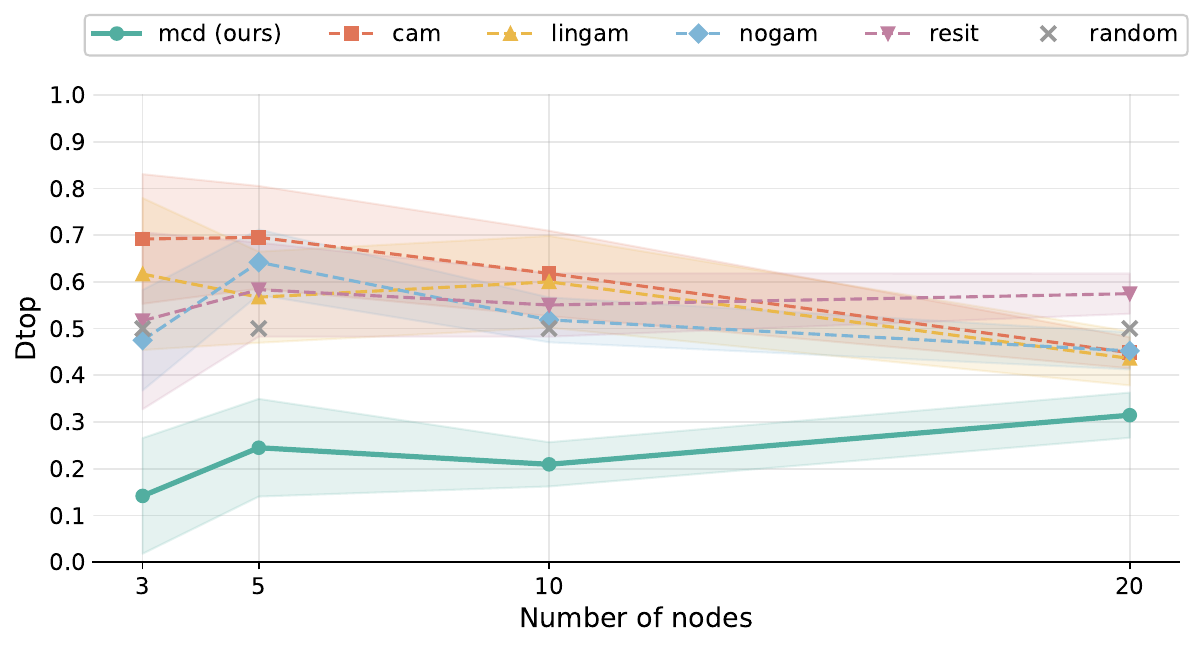}
    \caption{Mean $D_{\textnormal{top}}$ ($20$ seeds, the lower the better) performance on linear and nonlinear Gaussian data that are non-identifiable from pure observations. Error is $95\%$ confidence interval. Given three environments, MCD can infer the causal direction where all other state-of-the-art methods for observational causal discovery are no better than a baseline randomly selecting the causal order.\vspace{-1em}}
    \label{fig:main_experiments}
\end{figure}

\paragraph{Discussion of the experimental results.} In \cref{fig:main_experiments}, we consider the experimental performance of MCD on the synthetic data we described, given three environments. (Experiments with more environments are reported in  \cref{app:environments_ablation}.)  For the score estimation (which requires solving a kernel-ridge regression)  we use $0.001$ as weight of the $\ell_2$ regularization. We compare the performance of MCD with well-known and state of the art methods for causal discovery with additive noise models: CAM \citep{buhlmann2013cam}, RESIT \citep{peters2014causal}, LiNGAM \citep{shimizu2006alinear}, and NoGAM \citep{montagna23nogam} \footnote{For CAM and NoGAM we use the dodiscover implementations \citep{Li_Dodiscover_Causal_discovery} from \citet{montagna2023assumption}; RESIT and LiNGAM Python implementations are in the pip lingam package.}. Our findings are remarkable. While all other methods are no better than a baseline which uniformly samples a causal order at random, MCD achieves normalized $D_{\textnormal{top}}$ values (the lower, the better) ranging between $\sim 0.1$ for $3$ nodes, meaning the method perfectly infers the order over almost all seeds, to $\sim 0.3$ on $20$ nodes, remaining far better than random with statistical significance. We conclude the following:
\vspace{-.25em}
\begin{callout}
    \looseness-1The empirical results supports the theoretical claim that the number of environments doesn't need to scale with the number of nodes. Moreover, MCD's accuracy on graphs with up to $20$ nodes is a surprising achievement, given that scaling causal discovery with multiple environments beyond the bivariate setting is a well-known, unaddressed challenge, already found in \citet{monti2019nonsens,reizinger2023jacobianbased,montagna2026on}.
\end{callout}

\section{Conclusion}
\looseness-1The theory of our work shows that for structural causal models that are notoriously non-identifiable from observational data, \textit{only} three environments (with the lens of natural experiments, a control group and two treatment groups emerging in passively collected data) with invariant causal mechanisms are all it takes to uncover the direction of causal relations. Moreover, leveraging the acyclicity assumption, a common requirement in causality, we demonstrate that finding the causal graph is equivalent to solve a nonlinear ICA problem. This is  an impressive finding, knowing that general nonlinear ICA is a notoriously hard unsupervised learning problem (thought to be unsolvable until \citet{hyvarinen2016time}): our result is the first showing sufficient conditions for its solution  with a constant (and low, only $2$) number of auxiliary environments. Finally, we demonstrate that, when the causal model is bijective, having access to only three environments (i.e., naturally performing \textit{only} $2$ interventions on the latent sources) is sufficient to draw provably correct conclusions at the counterfactual level, the highest rung in the ladder of causality. Our MCD (Multienvironment Causal Discovery) method empirically validates the theoretical claims of graph identifiability on synthetic data. We leave for future work finding strong methodologies that leverage our theory for real-world causal discovery.

\section{Acknowledgments}
FM acknowledges the Chan Zuckerberg Initiative for financial support.



\bibliographystyle{plainnat}
\bibliography{biblio}
\newpage
\appendix

\section{Limitations}
The main limitation that we find in our work are of practical  nature. As discussed in the main paper, the assumption that all latent variables are subject to distribution shift can be strong for large graphs. Similarly, the assumption of causal sufficiency (no latent confounders) is a significant restriction, despite common in theoretical studies. Moreover, our MCD method is tested limited to synthetic data. Despite the promising performance, it would be interesting to have access to multi-environment real world data to evaluate its performance. 

\section{Related works}
Several works have tackled the problem of identifiability from multiple environments under the invariance principle, both in the causal discovery and in representation learning (ICA and causal representation learning). Below, we summarize the literature and how it relates to our works.

\paragraph{Observational causal discovery with SCMs.} Several works investigated the identifiability of causal graphs from pure observations: \citet{shimizu2006alinear,hoyer2008anm} focused on linear and nonlinear additive noise models; \citet{zhang2009pnl,immer2022on} on more or less slight generalizations. The burden of all these methodologies is that they require very strong restrictions on the class of structure causal models that must have generated the data.

\paragraph{Causal discovery with multiple environments.} Causal inference with the invariance principle was popularized by \citet{peters2015invariant} first (on linear models) and \citet{heinze2018invariant} later (on nonlinear additive noise models). Both methods address the problem of \textit{(i)} locally identifying the causal direction and \textit{(ii)} learn a regression model in the causal direction that is robust to distribution shifts due to interventions. The idea of their work is quite aligned to ours, despite being significantly more limited in terms of the class of SCMs they consider (additive noise models), while more expanded on the statistical analysis. In the context of causal discovery with structural causal models, \citet{rothenhausler2015backshift,ghassami2018multidomain,heurtebise2025identifiablemultiviewcausaldiscovery} study identifiability of causal graphs with multiple environments, under the assumption of linear models. \citet{montagna2026on} extends these ideas to arbitrary nonlinear models, with the requirements of Gaussian sources: for the causal discovery theory part, this is probably the work closest to ours. The papers of \citet{huang2020heterogeneous} and \citet{perry2022causal} are also concerned with causal discovery and multiple environments under some invariant principle: the main difference is that they do not assume a functional models of causality; as a result, they can only guarantee identifiability as the number of  auxiliary datasets  go to infinity. All these results can be viewed in the light of classic literature as causal discovery and inference under with natural experiments \citep{titiunik2021natural}. An interesting approach to causal discovery with some overlap to ours is \citet{mooij2011cyclic}: they show how to integrate interventional and observational information for causal discovery.

\paragraph{Causal discovery from interventions.} Several empirical and theoretical works investigates causal discovery with hard interventions. \citet{eberhardt2005on} seminal paper upper bounds the number of experiments (in terms of hard assignments) needed to recover the causal graph, in terms of the number of nodes. \citet{yang2018characterizing,brouillard2020differentiable,jaber2020causal} characterize equivalence classes of causal graphs one can learn from interventional data. From a practical perspective, \citet{brouillard2020differentiable,ke2023learning} introduce neural architectures for  causal discovery with interventions. 

\paragraph{Independent component analysis.} The ICA problem was introduce by \citet{comon1994ica} in their seminal paper. Beyond linearity, ICA was generally thought to be non-solvable; some narrow identifiability results have been presented in \citet{hyvarinen1999conformalica}---specifically, they showed that if the mixing function is conformal it can be recovered up to a rotation indeterminacy. But general nonlinear ICA identifiability has not been shown until \citet{hyvarinen2016time} with a multi-environment-based approach (in their work, time intervals serve as the environment label). These results sparked an entire lien of research at the intersection between ICA identifiability and the invariance principle \citep{hyvarinen2017nonlinear,hyvarinen2019nonlinear,gresele2019incomplete,khemakhem2020icebeem,khemakhem2020variational,halva2020hidden,halva2021snica,lachapelle22disentanglement}. Moreover, given that causal discovery with structural causal model is a simplified version of the ICA problem \citep{montagna2026on}, \citet{monti2019nonsens,reizinger2023jacobianbased} directly translated these findings showing that causal discovery can be solved with multiple environments: however, their results require a number of auxiliary datasets that scale at least linearly with the number of nodes in the graphs, making their guarantees not appealing, in practice. Moreover, their proposed methodologies fail to perform well on graphs with more than two nodes, even on synthetically generated data. 

\paragraph{Causal representation learning.} The field of causal representation learning \citep{scholkopf21towards} strongly focused on analyzing the assumptions under which theoretical guarantees of identifiability can be provided. Most of these results rely on variations of our multi environment setting (e.g., see \citet{squires2020active,buchholz2023learning,varici2025score}). \citet{yao2025unifying} nicely presents these works in a unified framework.


\section{Proof of theoretical results}
\subsection{Proof of Theorem \ref{thm:full_graph_identifiability}}\label{app:cd_proof}

We first formalize the sink-identification argument of \cref{example:bivariate_sink}.


\begin{proposition}[Sink identification from score ratios]\label{prop:sink_ratio}
Let $\x = \f(\s)$ for $\s \sim p_S^0$. For $e \in \{1,2\}$, define
$$
\Delta(0,e,\mathbf x) := \nabla \left(\log p^0(\mathbf x)-\log p^e(\mathbf x)\right),
\qquad
R_i(\mathbf x) := \frac{\Delta(0,1,\mathbf x)_i}{\Delta(0,2,\mathbf x)_i}.
$$
Under Assumptions \ref{ass:diffeomorphism}, \ref{ass:faithfulness}, and \ref{ass:sufficient_variability}, for every $i \in \set{1,...,d}$,
$$
X_i \textnormal{ is a sink } \iff R_i \indep \mathbf X_{-i}.
$$
\end{proposition}

\begin{proof}
Without loss of generality, let the variables be ordered according to a valid causal order, so that $J_{\f^{-1}}$ is lower triangular. By the change-of-variable formula,
$$
\Delta(0,e,\f(\s)) = J_{\f^{-1}}^T(\f(\s)) \nabla \left(\log p^0(\s)-\log p^e(\s)\right),
$$
hence
$$
\Delta(0,e,\f(\s))_i
=
\sum_{j=i}^d
\partial_{x_i} s_j \,
\partial_{s_j}\left(\log p^0(s_j)-\log p^e(s_j)\right).
$$
Assume first that $X_i$ is a sink. Then no node $X_j$, $j>i$, is a descendant of $X_i$. Equivalently, the $i$-th column of $J_{\f^{-1}}$ has only the diagonal nonzero entry, and therefore
$$
\Delta(0,e,\f(\s))_i
=
\partial_{x_i}s_i \,
\partial_{s_i}\left(\log p^0(s_i)-\log p^e(s_i)\right).
$$
Taking the ratio cancels the Jacobian term; defining $\mathbf g: \bfS \mapsto \bfR$, we get:
$$
R_i = \mathbf g_i(\s)
=
\frac{\partial_{s_i}\left(\log p^0(s_i)-\log p^1(s_i)\right)}
{\partial_{s_i}\left(\log p^0(s_i)-\log p^2(s_i)\right)}.
$$
Thus $R_i$ is a function of $S_i$ only. Since $X_i$ is a sink, $\mathbf X_{-i}$ is a function of $\mathbf S_{-i}$ only, and by independence of the sources we obtain
$$
R_i \indep \mathbf X_{-i},
$$
where we consider under $\bfS \sim p^{0}_{\bfS}$.

Conversely, suppose that $X_i$ is not a sink. By \cref{ass:sufficient_variability}, there exists $X_j \in \mathbf X_{-i}$ such that the vectors
$$
\big(\Delta(0,1,\f(\s))_i,\Delta(0,2,\f(\s))_i\big)
\quad \textnormal{and} \quad
\big(\partial_{s_j}\Delta(0,1,\f(\s))_i,\partial_{s_j}\Delta(0,2,\f(\s))_i\big)
$$
are not collinear on a positive-measure set. By \cref{ass:sufficient_variability}(i), $\Delta(0,2,\f(\s))_i \neq 0$ a.s. Hence
$$
\partial_{s_j} R_i(\f(\s))
=
\frac{
\partial_{s_j}\Delta(0,1,\f(\s))_i \, \Delta(0,2,\f(\s))_i
-
\Delta(0,1,\f(\s))_i \, \partial_{s_j}\Delta(0,2,\f(\s))_i
}{
\Delta(0,2,\f(\s))_i^2
}.
$$
The numerator is nonzero on a positive-measure set exactly because the above two vectors are not collinear. Therefore $\partial_{s_j} R_i(\f(\s)) \neq 0$ on a positive-measure set, so the structural equation defining $R_i$ depends on $S_j$.

Consider the structural equation models given by $\mathbf X := \f(\mathbf S)$ and $\mathbf R := \mathbf g(\mathbf S)$. Given that $\partial_{s_j} R_i(\f(\s)) \neq 0$ on a positive-measure subset of $\mathbf S$ domain, in the corresponding DAG we have $S_j \to R_i$; also, we have $S_j \to X_j$, so
$$
R_i \notindep_d X_j.
$$
By \cref{ass:faithfulness}, this implies
$$
R_i \notindep X_j.
$$
Since $X_j$ is a component of $\mathbf X_{-i}$, it follows that
$$
R_i \notindep \mathbf X_{-i}.
$$
This proves the claim.
\end{proof}

Next, we restate the theorem and present its demonstration.

\CdIdentifiability

\begin{proof}[Proof of Theorem 1]
By \cref{prop:sink_ratio}, the observed distributions $(p_{\mathbf X}^e)_{e=0}^2$ determine exactly which nodes are sinks: the quantity $R_i$ is observable from the score differences, and
$$
R_i \indep \mathbf X_{-i}
\iff
X_i \textnormal{ is a sink.}
$$

Now remove one sink. Since a sink has no children, the remaining variables again form an acyclic structural causal model with independent sources, obtained by deleting that node and its noise term. The restricted model still satisfies Assumptions \ref{ass:diffeomorphism}, \ref{ass:sufficient_variability}, and \ref{ass:faithfulness}. Applying \cref{prop:sink_ratio} recursively identifies the sinks of the reduced graph at each step. Therefore the data determine a set of valid causal orders of the SCM.

Under \cref{ass:faithfulness} of faithfulness, given the causal order, we can uniquely and correctly identify the causal graph. This proves the claim.
\end{proof}

\subsection{Proof of \cref{thm:ica_identifiability}}\label{app:ica_proof}
For convenience, we begin restating the theorem. 

\IcaIdentifiability

\subsubsection{Sketch of the proof}
Given that the proof is technical, first, we provide a sketch to present the intuition without mathematical complications.

\begin{proof}[Proof sketch]
By \cref{thm:full_graph_identifiability}, the causal order is identifiable. Hence, WLOG, we assume $J_{\f}$ is lower triangular. The proof proceeds by induction.

\paragraph{Base case.} For the source node $X_1$, the only nonzero entry in the row $J_{\f,1}$ is the diagonal element $J_{\f,11}$. Consider an alternative ICA solution $\widehat{\f}$: this must satisfy the same causal order imposed by $\f$: hence, we have $J_{\widehat{\f},11} \neq 0$ and all remaining entries in the row $J_{\widehat{\f},1}$ equal zero. Being both $\f, \widehat \f$ diffeomorphisms, there is a diffeomorphic function $\mathbf h: \widehat \f^{-1} \circ \f: \bfS \mapsto \widehat \bfS$: then, the structure of the Jacobians of $\f$ and $\widehat \f$ imposes that $\hat S_1 = \mathbf h_1(S_1)$, i.e., the noise of a source node $X_1$ is identified up to an element wise invertible transformation. 

\paragraph{Inductive step.} Let, for each $j < i$, $\hat S_j = \mathbf h_j(S_j)$. Given that $\mathbf h$ is lower triangular, we have $\hat S_i = \mathbf h_{i}(S_1,...,S_i)$. Now, by the inductive assumption, $S_j = \mathbf h_j^{-1}(\hat S_j)$ for $j<i$. Being $\hat S_i$ by construction independent of  $\hat S_1,...,\hat S_{i-1}$, equally 
\begin{equation*}
    \hat S_i \indep \bfS_{1:i-1}.
\end{equation*}
We then show that this forces $\mathbf h_i(S_1,\dots,S_i)$ to reduce to a function of $S_i$ alone, namely $\hat S_i = \mathbf h_i(S_i)$. This is the main technical step of the proof and requires care. Once established, the claim follows.
\end{proof}

Before presenting the complete proof of the main theorem, we introduce a useful lemma. 
\begin{lemma}[Uniqueness of a monotone scalar transport]\label{lem:monotone_scalar_transport}
Let $S$ be a scalar random variable with strictly positive density on $\mathcal I \subseteq \R$. Let $h, \tilde h : \mathcal I \to \R$ be two diffeomorphisms onto their images with the same orientation, i.e. either both strictly increasing or both strictly decreasing. If $h(S)$ and $\tilde h(S)$ have the same distribution, then
$$
h(s) = \tilde h(s), \qquad \forall s \in \mathcal I.
$$
\end{lemma}

\begin{proof}
Let $F$ be the cdf of $S$, and let $G$ be the common cdf of $h(S)$ and $\tilde h(S)$. Since $S$ has strictly positive density on $\mathcal I$, $F$ is strictly increasing on $\mathcal I$. Moreover, by the change of variable formula, the common law of $h(S)$ and $\tilde h(S)$ has strictly positive density on its support, so $G$ is strictly increasing there.

Assume first that $h$ and $\tilde h$ are strictly increasing. Then, for every $s \in \mathcal I$,
$$
G(h(s)) = \mathbb P(h(S) \leq h(s)) = \mathbb P(S \leq s) = F(s),
$$
and similarly
$$
G(\tilde h(s)) = \mathbb P(\tilde h(S) \leq \tilde h(s)) = \mathbb P(S \leq s) = F(s).
$$
Hence $G(h(s)) = G(\tilde h(s))$, and strict monotonicity of $G$ gives $h(s)=\tilde h(s)$.

If instead $h$ and $\tilde h$ are strictly decreasing, then for every $s \in \mathcal I$,
$$
G(h(s)) = \mathbb P(h(S) \leq h(s)) = \mathbb P(S \geq s) = 1-F(s),
$$
where the last equality uses that $S$ has a density. The same identity holds with $\tilde h$ in place of $h$, so again $G(h(s)) = G(\tilde h(s))$, and strict monotonicity of $G$ gives $h(s)=\tilde h(s)$.

This proves the claim.
\end{proof}

Next, we demonstrate \cref{thm:ica_identifiability}.

\begin{proof}
Let $(\widehat{\f}, \hat p_{\bfS}^e)_e$ be an alternative ICA model generating the same observed distributions $(p_{\bfX}^e)_e$. By \cref{thm:full_graph_identifiability}, the causal order is identifiable. Hence, after a common permutation of the coordinates, we may assume without loss of generality that both $J_{\f}$ and $J_{\widehat{\f}}$ are lower triangular.

Define
$$
\mathbf h := \widehat{\f}^{-1}\circ \f,
\qquad
\widehat{\bfS} = \mathbf h(\bfS).
$$
Since $\f$ and $\widehat{\f}$ are lower triangular diffeomorphisms, so is $\mathbf h$. Therefore, for each $i=1,\dots,d$,
$$
\hat S_i = \mathbf h_i(S_1,\dots,S_i).
$$

We now show by induction that each $\mathbf h_i$ depends only on $S_i$. For $i=1$, this is immediate. Assume that for some $i>1$, we already proved
$$
\hat S_j = \mathbf h_j(S_j), \qquad j=1,\dots,i-1.
$$
Then $\widehat{\bfS}_{1:i-1}$ is an invertible componentwise transformation of $\bfS_{1:i-1}$. Since the coordinates of $\widehat{\bfS}$ are independent, $\hat S_i \indep \widehat{\bfS}_{1:i-1}$, and therefore also
$$
\hat S_i \indep \bfS_{1:i-1}.
$$

On the other hand, $\hat S_i = \mathbf h_i(S_1, \ldots, S_i)$. Fix any value of $S_{1:i-1}$.
Because $h$ is a diffeomorphism, the map
$$
s_i \mapsto \mathbf h_i(S_1, \ldots, S_{i-1}, s_i)
$$
is a one-dimensional diffeomorphism, hence strictly monotone. Moreover, the domain of
$\mathbf h$ is connected, because it coincides with the domain of $\f$, which is connected by
assumption. Since $\partial_{s_i} \mathbf h_i$ is continuous and never vanishes, it must have
constant sign on the connected domain of $\mathbf h$. Therefore the above one-dimensional map
has the same orientation for every value of $S_{1:i-1}$.


Since $\hat S_i \indep \bfS_{1:i-1}$, the conditional law of $\hat S_i$ given
$\bfS_{1:i-1}$ coincides almost surely with the marginal law of $\hat S_i$.
Since $\hat S_i = \mathbf h_i(S_1,\ldots,S_i)$, there exists a set $E \subseteq \R^{i-1}$ of full measure under the law of $\bfS_{1:i-1}$ such that, for every $c \in E$, the random variable $\mathbf h_i(c,S_i)$ has the same
distribution as $\hat S_i$.

On the other hand, for fixed $c \in E$, the map
$$
s_i \mapsto \mathbf h_i(c,s_i)
$$
is strictly monotone, and by connectedness its orientation does not depend on
$c$. Hence, for any $c,c' \in E$, \cref{lem:monotone_scalar_transport} applies
to the two maps $s_i \mapsto \mathbf h_i(c,s_i)$ and
$s_i \mapsto \mathbf h_i(c',s_i)$, yielding
$$
\mathbf h_i(c,s_i)=\mathbf h_i(c',s_i), \qquad \forall s_i.
$$
Therefore there exists a function, still denoted by $\mathbf h_i$, such that
$$
\mathbf h_i(S_1,\ldots,S_{i-1},s_i)=\mathbf h_i(s_i)
$$
for every $s_i$ and every $(S_1,\ldots,S_{i-1}) \in E$.

Finally, since $E$ has full measure under the law of $\bfS_{1:i-1}$, it is dense
in $\supp(\bfS_{1:i-1})$. Since the original map
$\mathbf h_i(S_1,\ldots,S_{i-1},s_i)$ is continuous, the above identity extends
to all values of $S_{1:i-1}$ in $\supp(\bfS_{1:i-1})$. Thus
$$
\hat S_i = \mathbf h_i(S_i).
$$

Repeating the argument for all $i=1,\dots,d$, we obtain
$$
\hat S_i = \mathbf h_i(S_i), \qquad i=1,\dots,d.
$$
Since $\mathbf h$ is a diffeomorphism, each $\mathbf h_i$ is invertible. Therefore,
$$
\widehat{\f}^{-1}\circ \f(\bfS) = (h_1(S_1),\dots,h_d(S_d)),
$$
so $\f$ is identifiable up to elementwise reparametrization of the sources.
\end{proof}

\subsection{Proof of \cref{thm:counterfactual}}\label{app:cor_proof}

In this section, we present teh corollary statement and its proof. 

\CorCounterfactual
\begin{proof}
Let $(\widehat{\f}, \hat p_{\bfS}^e)_e$ be an alternative SCM compatible with the same observational distribution and the same two auxiliary environments. By \cref{thm:ica_identifiability}, the two models can differ only by an elementwise reparametrization of the sources. That is, there exists an invertible componentwise map $\mathbf h$ such that
$$
\widehat{\bfS} = \mathbf h(\bfS),
\qquad
\widehat{\f} = \f \circ \mathbf h^{-1}.
$$
Thus, the two SCMs differ only by a relabeling of the exogenous variables.

Fix now any factual observation $\mathbf x$. In the original model, abduction gives $\mathbf s = \f^{-1}(\mathbf x)$. In the alternative model, abduction gives
$$
\hat{\mathbf s} = \widehat{\f}^{-1}(\mathbf x) = \mathbf h(\mathbf s).
$$
Hence the inferred exogenous variables in the two models correspond exactly under the same componentwise transformation.

Now write the original SCM as
$$
X_i := F_i(X_{\Parents_i}, S_i), \qquad i = 1, \dots, d.
$$
Since $\widehat{\f} = \f \circ \mathbf h^{-1}$ and $\mathbf h$ acts componentwise, the alternative SCM admits the representation
$$
X_i := F_i(X_{\Parents_i}, \mathbf h_i^{-1}(\widehat S_i)), \qquad i = 1, \dots, d.
$$
Therefore, after abduction at the factual observation $\mathbf x$, the two models satisfy
$$
\mathbf h_i^{-1}(\hat s_i) = s_i, \qquad i = 1, \dots, d.
$$

Consider now any intervention on the endogenous variables. This replaces the same structural equations in both SCMs by the same intervened assignments. We claim that the resulting counterfactual values of $\mathbf X$ coincide in the two models.

Since the SCM is acyclic, there exists a causal order. Without loss of generality, assume that this order is $1, \dots, d$. We now argue by induction on $i$.

For the base case $i=1$, if $X_1$ is directly intervened on, then its counterfactual value is the same in both models by definition of the intervention. Otherwise,  if intervention occurs on $X_j$ for $j \neq i$, since $X_1$ has no parents, its counterfactual value in the original model is given by
$$
X_1 := F_1(S_1),
$$
while in the alternative model it is given by
$$
X_1 := F_1(\mathbf h_1^{-1}(\widehat S_1)).
$$
After abduction, $\mathbf h_1^{-1}(\hat s_1) = s_1$, so the two right-hand sides are equal. Hence the counterfactual value of $X_1$ is the same in both models.

For the inductive step, let $i > 1$ and assume that the counterfactual values of $X_1, \dots, X_{i-1}$ coincide in the two models. If $X_i$ is directly intervened on, then again its counterfactual value is the same in both models by definition of the intervention. Otherwise, in the original model its counterfactual value is obtained from
$$
X_i := F_i(X_{\Parents_i}, S_i),
$$
while in the alternative model it is obtained from
$$
X_i := F_i(X_{\Parents_i}, \mathbf h_i^{-1}(\widehat S_i)).
$$
By the inductive hypothesis, the counterfactual values of all variables in $X_{\Parents_i}$ coincide in the two models, since $\Parents_i \subseteq \{1, \dots, i-1\}$. Moreover, after abduction, $\mathbf h_i^{-1}(\hat s_i) = s_i$. Therefore the right-hand sides of the two structural equations are equal, and so the counterfactual value of $X_i$ is the same in both models.

By induction, the counterfactual values of all coordinates of $\mathbf X$ coincide in the two models. Hence, the two SCMs produce the same counterfactual value of $\mathbf X$ for the given factual observation $\mathbf x$ and the given intervention.

Since $\mathbf x$ and the intervention were arbitrary, every SCM compatible with the observational distribution and the two auxiliary environments gives the same counterfactual answers. This proves counterfactual identifiability.
\end{proof}


\section{On the genericity of the faithfulness assumption}\label{app:faithfulness}

Faithfulness assumption is well known to be generic in precise terms both in parametric \citep{Spirtes2000,uhler2012geometry} and non-parametric Bayesian networks \citep{boeken2026bayesian}. Below, we provide an example of genericity for the structural equation model (SEM) of \cref{example:bivariate_sink}. Note that such an SEM defines a Bayesian network, so all genericity results in the literature apply to our problem. Computations below only serve for illustrative purposes.

\subsection{Faithfulness in Linear Gaussian Structural Equation Models}\label{app:linear_gauss}

We now formally state that the statistical dependence $R_1 \notindep X_2$, motivated by the d-connection in Figure \ref{fig:measurement_model}, holds generically. In linear Gaussian settings, violations of this faithfulness assumption correspond to a set of environment parameters with Lebesgue measure zero, which is well known for faithfulness on Linear Gaussian SCMs. 

\begin{theorem}[Generic Faithfulness of the Score Ratio]
    Assume a bivariate linear Gaussian structural causal model $X_1 = S_1$ and $X_2 = aS_1 + bS_2$ with $b \neq 0$. Let the sources $S_i$ across environments $e \in \{0, 1, 2\}$ be distributed as $\mathcal{N}(\mu_i^e, (\sigma_i^e)^2)$. Let $R_1$ be the score ratio defined in Equation \eqref{eq:example_ratios}. The set of environment parameters $\theta = \{ \mu_i^e, \sigma_i^e \}_{i,e}$ for which $R_1$ and $X_2$ are statistically independent has Lebesgue measure zero.
\end{theorem}

\begin{proof}
    Because the sources are Gaussian, their log-densities are quadratic. The score difference for source $i$ between the base environment $0$ and auxiliary environment $e$ is therefore a strictly linear function of $s_i$:
    $$
    \Delta_{ie}(s_i) = \partial_{s_i} \left( \log p^0(s_i) - \log p^e(s_i) \right) = \alpha_{ie} s_i + \beta_{ie}
    $$
    where $\alpha_{ie} = \frac{1}{(\sigma_i^e)^2} - \frac{1}{(\sigma_i^0)^2}$ and $\beta_{ie} = \frac{\mu_i^0}{(\sigma_i^0)^2} - \frac{\mu_i^e}{(\sigma_i^e)^2}$. 
    
    From the causal mechanism $x_2 = as_1 + bs_2$, we have $s_2 = \frac{x_2 - ax_1}{b}$. The Jacobian term is thus the constant $\partial_{x_1} s_2 = -a/b$. Let us denote this constant as $\gamma$. 
    
    Substituting these linear forms into the definition of $R_1$, we obtain a rational function of $S_1$ and $S_2$:
    $$
    R_1(S_1, S_2) = \frac{\alpha_{11} S_1 + \beta_{11} + \gamma (\alpha_{21} S_2 + \beta_{21})}{\alpha_{12} S_1 + \beta_{12} + \gamma (\alpha_{22} S_2 + \beta_{22})}
    $$
    We wish to find the conditions under which $R_1(S_1, S_2) \perp X_2$, where $X_2 = aS_1 + bS_2$. For a non-constant rational function of joint Gaussians to be perfectly independent of a non-trivial linear combination of those same Gaussians, the parameter gradients must align such that $R_1$ depends on $S_1$ and $S_2$ strictly through vectors orthogonal to $X_2$. 
    
    This requirement imposes strict algebraic constraints on the parameter space $\theta$. Specifically, the coefficients of the numerator and denominator of $R_1$ must satisfy a non-trivial polynomial equality involving $\alpha_{ie}, \beta_{ie}, a,$ and $b$. 
    
    Let $\Theta \subset \mathbb{R}^k$ be the parameter space of all valid environment means and variances. The subset of parameters that satisfy this polynomial constraint defines a proper algebraic variety. By well-known results in algebraic geometry \citep{Spirtes2000}, the roots of a non-trivial multivariate polynomial evaluated over $\mathbb{R}^k$ have a Lebesgue measure of zero. 
    
    Consequently, exact cancellations leading to $R_1 \perp X_2$ occur only on a set of measure zero, concluding the proof.
\end{proof}


\section{On the genericity of sufficient variability}\label{app:sufficient_variability}

\subsection{Finite-dimensional analytic families}

\begin{proposition}[Genericity of \cref{ass:sufficient_variability}]\label{prop:generic_sufficient_variability}
Let $\Theta \subseteq \R^m$ be open, and a parametrized SCM
$$
(\f_\theta, p_{\bfS,\theta}^e)_{e=0}^2, \qquad \theta \in \Theta.
$$
Assume that $\mathcal S  = \supp (\bfS)$ is connected and open, and that for every $i \in \set{1,...,d}$ and $e \in \set{0,1,2}$ the density $p_{S_i,\theta}^e$ is strictly positive on its domain. Further, assume that the maps
$$
(s_i,\theta) \mapsto p_{S_i,\theta}^e(s_i), \qquad (\s,\theta) \mapsto \f_\theta(\s)
$$
are real analytic. If there exists $\theta^\star \in \Theta$ such that \cref{ass:sufficient_variability} holds, then the set of $\theta \in \Theta$ for which \cref{ass:sufficient_variability} fails has Lebesgue measure zero.
\end{proposition}

\begin{proof}
For $i \in \set{1,...,d}$, define
$$
G_i(\s,\theta) := \Delta(0,2,\f_\theta(\s))_i
$$
where 
$$
\Delta(0, e, \f_{\theta}(\s)) := \nabla \left(\log p_{\bfX,\theta}(\f_\theta(\s))^0 - \log p_{\bfX,\theta}(\f_\theta(\s))^e\right).
$$

For $i,j \in \set{1,...,d}$, $j \neq i$, define
$$
H_{ij}(\s,\theta)
:=
\Delta(0,1,\f_\theta(\s))_i \, \partial_{s_j}\Delta(0,2,\f_\theta(\s))_i
-
\Delta(0,2,\f_\theta(\s))_i \, \partial_{s_j}\Delta(0,1,\f_\theta(\s))_i.
$$
Since positive real analytic functions have real analytic logarithm, the maps
$$
(s_i,\theta) \mapsto \partial_{s_i}\left(\log p^0_{S_i,\theta}(s_i)-\log p^e_{S_i,\theta}(s_i)\right)
$$
are real analytic. Therefore $G_i$ and $H_{ij}$ are real analytic in $(\s,\theta)$.

Since $p_{\bfS,\theta}^0$ is strictly positive on $\mathcal S$, \cref{ass:sufficient_variability}(i) fails for node $i$ if and only if $G_i(\cdot,\theta)$ vanishes on a positive-measure subset of $\mathcal S$. By analyticity on the connected open set $\mathcal S$, this is equivalent to
$$
G_i(\cdot,\theta) \equiv 0.
$$
Likewise, for fixed $i$ and $j \neq i$, the two vectors in \cref{ass:sufficient_variability}(ii) are collinear if and only if $H_{ij}(\s,\theta)=0$. Hence \cref{ass:sufficient_variability}(ii) fails for a non-sink $X_i$ if and only if
$$
H_{ij}(\cdot,\theta) \equiv 0, \qquad \forall j \neq i.
$$

Fix $\theta^\star$ as in the statement. Then $G_i(\cdot,\theta^\star) \not\equiv 0$ for every $i$. Moreover, for every non-sink $X_i$, there exists $j_i \neq i$ such that $H_{ij_i}(\cdot,\theta^\star) \not\equiv 0$. Fix $\s_0 \in \mathcal S$. Expanding around $\s_0$,
$$
G_i(\s,\theta) = \sum_{\alpha} c_{i,\alpha}(\theta) (\s-\s_0)^\alpha,
\qquad
H_{ij_i}(\s,\theta) = \sum_{\alpha} d_{i,\alpha}(\theta) (\s-\s_0)^\alpha,
$$
where each coefficient is real analytic in $\theta$. Since $G_i(\cdot,\theta^\star)$ and $H_{ij_i}(\cdot,\theta^\star)$ are not identically zero, there exist multi-indices $\alpha_i,\beta_i$ such that $c_{i,\alpha_i}$ and $d_{i,\beta_i}$ are non-trivial real analytic functions on $\Theta$. Therefore their zero sets have Lebesgue measure zero.

If \cref{ass:sufficient_variability} fails at $\theta$, then either $G_i(\cdot,\theta)\equiv 0$ for some $i$, or $H_{ij_i}(\cdot,\theta)\equiv 0$ for some non-sink $X_i$. Hence the exceptional set is contained in the finite union
$$
\bigcup_{i=1}^d \set{\theta \in \Theta \,:\, c_{i,\alpha_i}(\theta)=0}
\;\cup\;
\bigcup_{i \,:\, X_i \textnormal{ non-sink}} \set{\theta \in \Theta \,:\, d_{i,\beta_i}(\theta)=0},
$$
which has Lebesgue measure zero.
\end{proof}

\subsection{Linear Gaussian illustration}\label{app:linear_gauss_sufficient_variability}

\begin{proposition}[Linear Gaussian illustration]\label{prop:linear_gaussian_sufficient_variability}
Consider the bivariate setting of \cref{example:bivariate_sink}, and assume
$$
X_1 := S_1, \qquad X_2 := a X_1 + S_2,
$$
with $a \in \R$, and
$$
S_i^e \sim \mathcal N(0,(\sigma_i^e)^2), \qquad i=1,2,\ e=0,1,2.
$$
Define
$$
\alpha_i^e := \frac{1}{(\sigma_i^e)^2} - \frac{1}{(\sigma_i^0)^2}.
$$
If
$$
\alpha_1^2 \neq 0, \qquad \alpha_2^2 \neq 0, \qquad \alpha_2^1 \alpha_1^2 - \alpha_2^2 \alpha_1^1 \neq 0,
$$
then the set of coefficients $a$ for which \cref{ass:sufficient_variability} fails is $\set{0}$.
\end{proposition}

\begin{proof}
For $e \in \set{1,2}$,
$$
\partial_{s_i}\left(\log p^0(s_i)-\log p^e(s_i)\right) = \alpha_i^e s_i,
$$
hence
$$
\Delta(0,e,\x)_2 = \alpha_2^e s_2,
\qquad
\Delta(0,e,\x)_1 = \alpha_1^e s_1 - a \alpha_2^e s_2.
$$
Therefore \cref{ass:sufficient_variability}(i) holds for both $i=1,2$.

The only non-sink is $X_1$, and in \cref{ass:sufficient_variability}(ii) we can take $j=2$. Since
$$
\partial_{s_2}\Delta(0,e,\x)_1 = - a \alpha_2^e,
$$
the determinant testing collinearity is
$$
\Delta(0,1,\x)_1 \, \partial_{s_2}\Delta(0,2,\x)_1
-
\Delta(0,2,\x)_1 \, \partial_{s_2}\Delta(0,1,\x)_1
=
a \left(\alpha_2^1 \alpha_1^2 - \alpha_2^2 \alpha_1^1\right) s_1.
$$
By assumption, this vanishes on a positive-measure set if and only if $a=0$. Hence \cref{ass:sufficient_variability}(ii) holds for every $a \neq 0$, and fails for $a=0$.
\end{proof}


\section{Algorithm and additional experimental results}
In this section we present further details on the MCD algorithm, details on the compute resources, and additional experimental results.

\subsection{The MCD algorithm}\label{app:algorithm}
The MCD algorithm is introduced in \cref{sec:algorithm}, \cref{alg:mcd}. Here we provide further details on statistical estimation procedures and computational complexity.

\paragraph{Statistical inference steps: the score matching approach.} The key object to infer from the data is the difference of the score between different environments, i.e., the quantity $\Delta(0, e, \x)$ of \cref{eq:deltas}. Then, the main statistical inference task is estimating partial derivatives of joint log-densities. This is a well-studied problem, and can be addressed with an umbrella of techniques that go under the name of score matching \citep{hyvarinen2005estimation} ($\nabla \log p$ is known as the score function). In this work, we adopt the \textit{Stein estimator} introduced by \citet{li2018gradient} (to which we point the interested reader for a thorough treatment of Stein estimation techniques). For the independence testing required to identify a sink node at each iteration of the algorithm, we use the HSIC test \citep{gretton2007hsic}, also a standard choice in causality \citep{zhang2012kernel}. Both the score estimator and the independence testing are correct in the population limit. The same property is hence inherited by MCD.

\paragraph{Algorithmic complexity.}  The method recursively iterates on all nodes to allocate them a position in the estimated causal order. Each iteration involves running the gradient $\nabla \log p_{\bfX}$ estimator for each input environment; this which requires solving a kernel ridge regression problem with cubic complexity in the number of nodes. Overall, for $\mathcal E$ environments, $d$ nodes in the graph, and $n$ number of samples, the score estimation requires $\mathcal O(\mathcal E dn^3)$ operations. The HSIC independence testing requires $\mathcal O(n^2d)$ operations, and is run for each node, contributing $\mathcal O(d^2n^2)$. Overall, MCD has $\mathcal O(\mathcal E dn^3 + d^2n^2)$. 

 \begin{figure}
     \centering
     \includegraphics[width=0.65\linewidth]{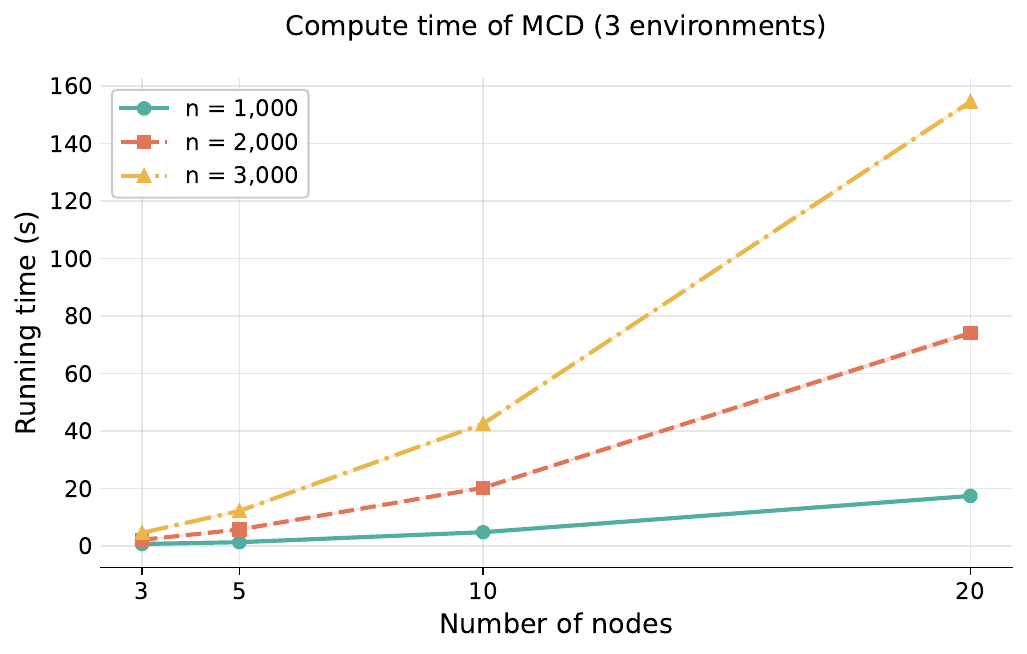}
     \caption{Average ($20$ seeds) execution time of the MCD algorithm, per number of samples and number of nodes. Experiments are run on a laptop Lenovo ThinkPad T14 Gen 5 on CPU.}
     \label{fig:execution_time}
 \end{figure}

\subsection{Additional experimental results}\label{app:experiments}
In this section, we present additional experimental results on the MCD method, in support of our theory.

\subsubsection{Statistical efficiency of MCD}\label{app:statistical_efficiency}
We empirically evaluate the statistical efficiency of the MCD algorithm. In \cref{fig:statistical_efficiency} we show the average topological order divergence on $3$ to $20$ nodes of our method, achieved when each base or auxiliary environment is assigned $1000,2000$ or $3000$ samples. We observe meaningful improvements when increasing datasets size from $1000$ to $2000$; interestingly, after $2000$ samples the quality of inference is substantially stable. 

\begin{figure}
    \centering
    \includegraphics[width=0.65\linewidth]{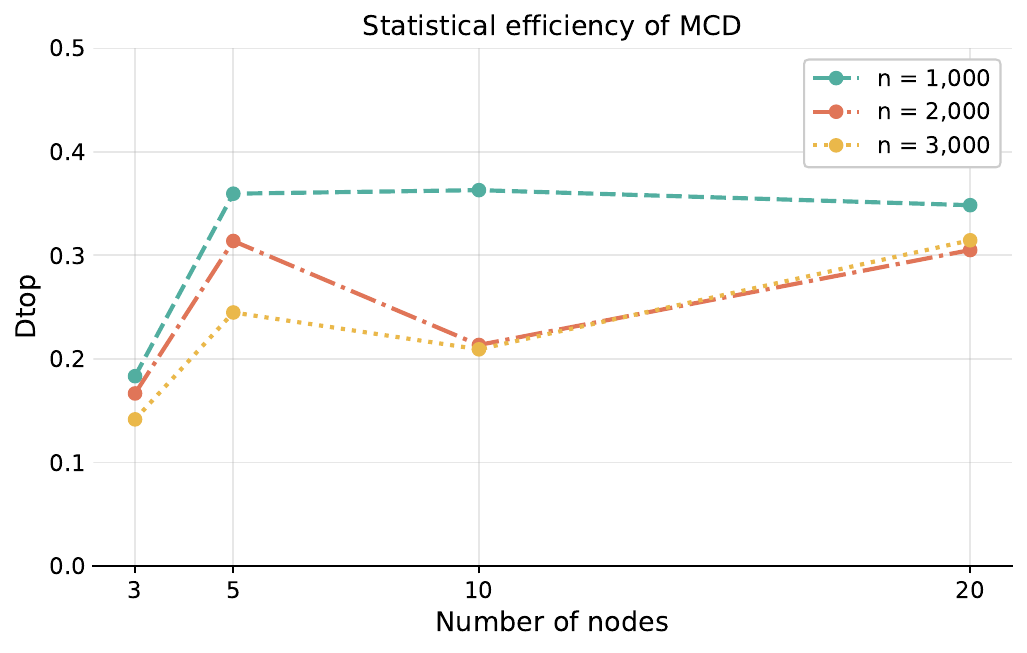}
    \caption{Average normalized $D_{\textnormal{top}}$ (the lower, the better) and $95\%$ confidence interval ($20$ seeds) of MCD at varying datasets size. $n$ refers to the number of samples observed in each environment. We see that for $n=1000$ MCD already infers causality better than random (expected accuracy at $0.5$); for $2000$ samples MCD significantly increases the accuracy, which remains almost stable for $n=3000$.}
    \label{fig:statistical_efficiency}
\end{figure}

\subsubsection{$3$ environments is all it takes}\label{app:environments_ablation}
In \cref{fig:environments_ablation} we evaluate MCD accuracy when we add auxiliary information from more environments. We consider settings where inference occurs on $3,5,7$ datasets: we observe that adding environments beyond $3$ does not affect the accuracy with statistical significance (sometimes leading to marginal improvements, sometimes to marginal degradation of the mean normalized $D_{\textnormal{top}}$). This is in line with our main claim in \cref{thm:full_graph_identifiability}, stating that \textit{only} two auxiliary environments are sufficient for causal graph discovery.

\begin{figure}
    \centering
    \includegraphics[width=0.7\linewidth]{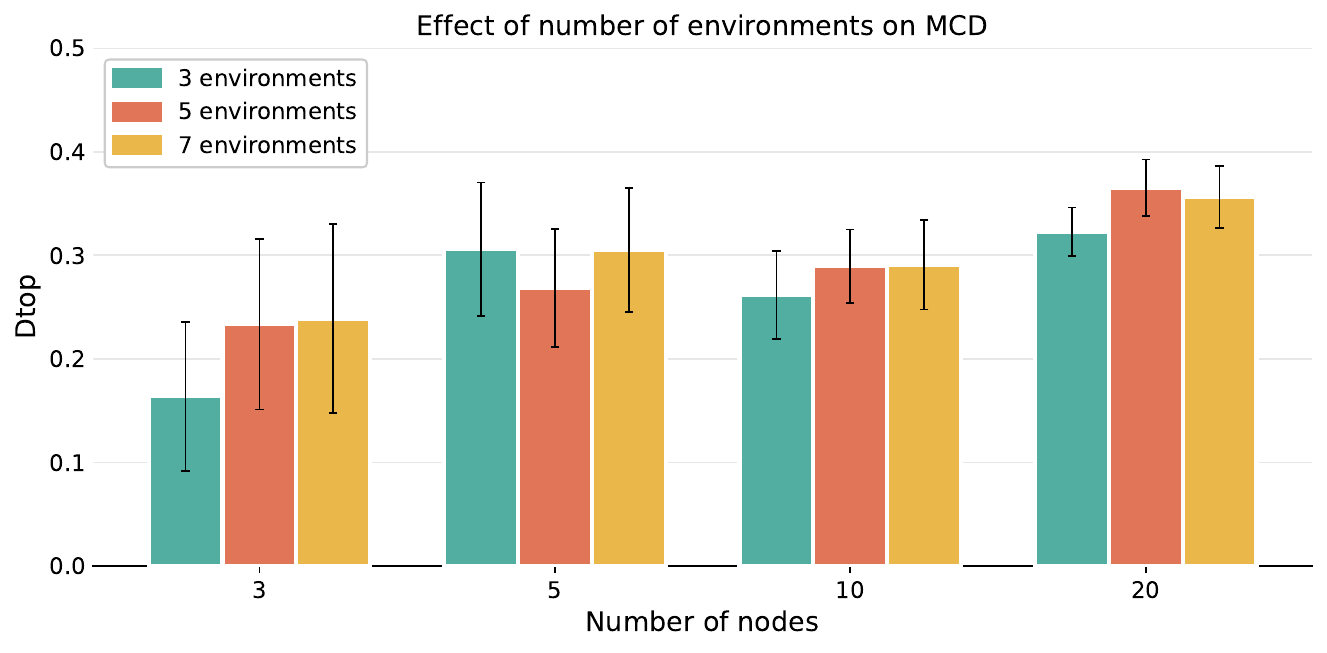}
    \caption{We compare MCD  average normalized $D_{\textnormal{top}}$ (the lower, the better) when inference occurs on $3,5,7$ environments; in line with the findings of \cref{thm:full_graph_identifiability}, $3$ datasets are sufficient for causal discovery. In fact, up to finite samples effects, adding environments does not clearly improves inference accuracy. Error bars denote $95\%$ confidence intervals over $20$ seeds.}
    \label{fig:environments_ablation}
\end{figure}

\subsubsection{Ablation study on the mechanism type}\label{app:mechanisms_ablation}
We analyse the impact of the mechanism choice on the results of \cref{fig:main_experiments} in \cref{sec:experiments}. In particular, we run inference with $3$ environments, each with $3000$ samples, and repeat each experimental run on $20$ seeds. The \textit{red} curve denotes MCD accuracy achieved on data generated by SCMs with linear mechanisms only. The  \textit{green} curve, similarly, is the accuracy when mechanisms are only nonlinear, randomly sampled from a neural network (as described in \cref{sec:experiments}). Remarkably, adding (severe) nonlinearity in the data generating process does not affect MCD ability to infer causality.

\begin{figure}
    \centering
    \includegraphics[width=0.65\linewidth]{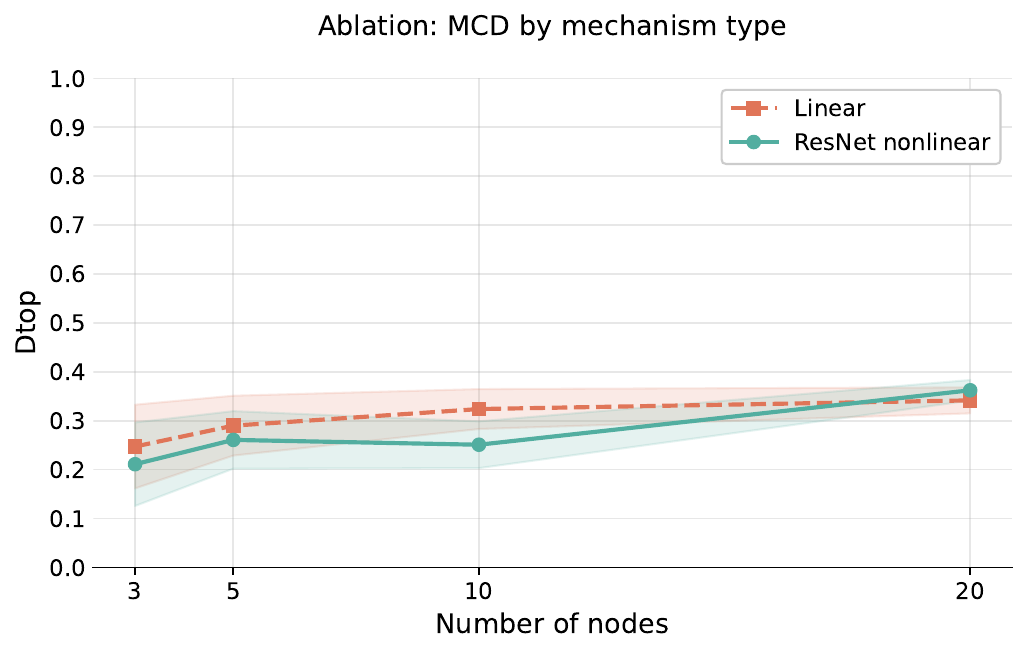}
    \caption{We study MCD average $D_{\textnormal{top}}$ (the lower, the better) separately, on synthetic datasets whose causal mechanisms are \textit{only} linear or \textit{only} nonlinear (parametrized with a ResNet, as described in \cref{sec:experiments}). We find that MCD accuracy is substantially unaffected by the nonlinearity, as predicted by the theory. Error bars denote $95\%$ confidence intervals over $20$ seeds.}
    \label{fig:mechanisms_ablation}
\end{figure}

\end{document}